\documentclass[letterpaper, 10 pt, conference]{ieeeconf}
\pdfoutput=1
\IEEEoverridecommandlockouts
\usepackage{cite}

\usepackage[pagewise]{lineno}
\usepackage{graphicx}
\usepackage[margin = 1in]{geometry}
\usepackage{graphicx}
\usepackage{mathtools,array,amssymb,amsmath}

\usepackage{amsthm}
\usepackage{breqn}
\newcolumntype{L}{>{$}l<{$}} % math-mode version of "l" column type
\usepackage{amstext} % for \text macro
\graphicspath{{./figures/}}
\newtheorem{definition}{Definition}
\newtheorem{theorem}{Theorem}
\newtheorem{lemma}{Lemma}
\newtheorem{proposition}{Proposition}

\usepackage{xcolor}

\newcommand{\remove}[1]{}

\allowdisplaybreaks

\usepackage{hyperref}
\newcommand{\savefootnote}[2]{\footnote{\label{#1}#2}}
\newcommand{\repeatfootnote}[1]{\textsuperscript{\ref{#1}}}

\title{\LARGE \bf
Provable Probabilistic Safety and Feasibility-Assured Control for Autonomous Vehicles using Exponential Control Barrier Functions
}

\author{Spencer Van Koevering$^{1}$, Yiwei Lyu$^{1}$, Wenhao Luo$^{2}$, John Dolan$^{1}$% <-this % stops a space
	\thanks{$^{1}$Spencer Van Koevering, Yiwei Lyu, and John Dolan are with Carnegie Mellon University,
		Pittsburgh PA 15213, USA. {\tt\small\{yiweilyu,jdolan\}@andrew.cmu.edu, vankoesd@whitman.edu}}%
	\thanks{$^{2}$Wenhao Luo is with the Department of Computer Science, University of North Carolina at Charlotte, Charlotte NC 28223, USA. Email: {\tt \small wenhao.luo@uncc.edu}}%
}

\begin{document}
% \linenumbers % Uncomment this to enable line numbers in the peer review

\maketitle
\thispagestyle{empty}
\pagestyle{empty}

%%%%%%%%%%%%%%%%%%%%%%%%%%%%%%%%%%%%%%%%%%%%%%%%%%%%%%%%%%%%%%%%%%%%%%%%%%%%%%%%
\begin{abstract}

With the increasing need for safe control in the domain of autonomous driving, model-based safety-critical control approaches are widely used, especially Control Barrier Function (CBF)-based approaches. Among them, Exponential CBF (eCBF) is particularly popular due to its realistic applicability to high-relative-degree systems. However, for most of the optimization-based controllers utilizing CBF-based constraints, solution feasibility is a common issue arising from potential conflict among different constraints. Moreover, how to incorporate uncertainty into the eCBF-based constraints in high-relative-degree systems to account for safety remains  an open challenge. In this paper, we present a novel approach to extend an eCBF-based safe critical controller to a probabilistic setting to handle potential motion uncertainty from system dynamics. More importantly, we leverage an optimization-based technique to provide a solution feasibility guarantee in run time, while ensuring probabilistic safety. Lane changing and intersection handling are demonstrated as two use cases, and experiment results are provided to show the effectiveness of the proposed approach.

\end{abstract}

%%%%%%%%%%%%%%%%%%%%%%%%%%%%%%%%%%%%%%%%%%%%%%%%%%%%%%%%%%%%%%%%%%%%%%%%%%%%%%%%
\section{INTRODUCTION}

Human-robot coexistence in the future brings a critical need to explore the way to ensure robot safe control. 
Classical control methods including proportional–integral–derivative (PID) control are widely used, but are not good at enforcing system constraints. 
Consider the adaptive cruise control (ACC) problem as an example. The system attempts to follow a desired velocity while maintaining a critical safe distance from the leading vehicle at the same time. Classic control would not be the best choice due to the additional safety constraint \cite{PIDstuff, ACC}.

There are multiple ways to approach safety in control systems. At the planner level, safety can be enforced using hard constraints or encouraged via optimality, while low-level controllers can use a reactive safety approach\cite{wang2020infusing}. Safe set methods provide constraints for low-level controllers that guarantee safety, regardless of the planning that determines high-level control. Human motion patterns are often complicated and cannot easily be modelled\cite{safetyHRI}. This makes safe set methods, including CBFs, very appealing methods for safe control in the presence of humans, as they can be independent of the intention of other agents in the system. In addition, CBFs are relatively light and scalable when compared to other safe set-based methods, like reachability analysis \cite{li2021comparison}. 

As a modern control approach, Control Barrier Functions (CBF) \cite{ames2019control, Ames_2017} offers a more admissive control space based on the idea of set forward invariance, which is formally provable. Therefore, it is seeing increased usage in safe control \cite{lyu2021probabilistic, 7524935, wang2017safe, lyu2022adaptive}. When it comes to complex systems, Exponential CBF (eCBF) demonstrates its usefulness in high-relative-degree safety-critical control\cite{7524935, ames2019control}. Due to this, and its resilience in situations with other unpredictable agents, it is a common choice for safety-critical controllers for automated vehicles.

In this paper, we propose a novel probabilistic eCBF-based control framework for higher-relative-degree systems. Our \textbf{main contributions} in this paper are: 1) a novel extension of exponential Control Barrier Function-based constraints to a probabilistic setting for stochastic system dynamics while preserving provably correct safety guarantees; 2) a time-varying CBF parameter selection technique for a point-wise feasibility guarantee, enhancing the general applicability to complex systems with higher relative degree; 3) Two case studies as application examples in the domain of autonomous driving, and the use of numerical simulation to demonstrate the validity and effectiveness of the proposed approach.

\section{Related Work}
Safe control is an area of automated control in which there are states that are critical to avoid, unsafe states. Common approaches to safety-critical control include safety indexes, artificial potential fields (APFs), and control barrier functions. One safe set approach is the safety index. A safety index is a function of the state of a system, $\phi: X\to \mathbb{R}$ such that $\phi(\textbf{x}) \leq 0$ if and only if $\textbf{x}$ is a safe state \cite{safetyHRI, liu2017designing}. Furthermore, $\dot{\phi}(\textbf{x}) \leq 0$ if $\textbf{x}$ is safe\cite{safetyHRI}. However, this is a stronger condition than needed, and it is safe for $\phi(x)$ to increase as long as it never becomes positive\cite{ames2019control}. An APF works by having a goal region that asserts an attractive potential, while having obstacles assert a repulsive potential\cite{singletary2020comparative}. 

CBFs are similar to a general safety index method but allow for a larger set of safe inputs by only enforcing the safety level set\cite{ames2019control}. APFs can be formulated as CBFs and can be outperformed by CBFs in terms of safety\cite{singletary2020comparative}. CBFs are a highly general and flexible method for safe control that include several variants like the eCBF, which allows for the same guarantees as a CBF in higher-relative-degree systems. CBFs are great for dealing with complicated systems, including high-dimensional and high-degree problems\cite{safetyHRI, 8619142, li2021comparison}. CBF validity and applicability to practical higher-relative-degree problems like adaptive cruise control and lane changing are well known \cite{ames2019control, he2021rulebased, Ames_2017, 7040372}.

There are many existing works addressing safety-critical control in deterministic environments \cite{he2021rulebased, Ho-2020-126658}. However, in a more realistic setting, how to deal with the existence of uncertainty remains a problem. There are some efforts on incorporating uncertainty into CBF-based safety constraints formulation \cite{9655206, 9196757, luo2020multi}. For high-relative-degree systems, learning-based methods are often used to model the uncertainty. Instead of incorporating this information into CBF-based safety constraints, usually it is used to obtain a learnt distribution over the system dynamics. So, generally, they allow for learning the dynamics model but do not provide a provable safety guarantee. Another approach is that of the robust CBF. Robust CBFs are a method by which safety can be guaranteed under limited non-linearity or bounded uncertainty\cite{buch2021robust, 7526114}. This a provable guarantee like what we offer, excepting that the robust guarantee is not probablistic. This is possible because the uncertainty must be bounded in order to use a robust CBF \cite{buch2021robust, 7526114}, whereas we consider unbounded uncertainty. 

Moreover, for optimization-based safety-critical controllers that use CBF-based constraints, solution feasibility is a common problem.  \cite{lyu2021probabilistic} mentions that the solution feasibility can be guaranteed by assuming that, in the worst case, making all robots decelerate to zero velocity immediately at the next time step can always prevent collision, and therefore the feasible solution space will always be non-empty. However, a more principled scheme with explicit theoretical grounding is desired to automatically decide whether the vehicle needs to perform full braking before it is too late. 

% Unlike these previous approaches,
We offer a provable probabilistic safety and feasibility guarantee for high-relative-degree systems under unbounded uncertainty. For example, safety-critical control design using control barrier functions when applied to an autonomous lane change model requires higher-degree barriers for constraints on position, as it models steering and acceleration as input. The extension of the framework for probabilistic barriers and feasibility guarantees to higher-degree systems would show the effectiveness of this strategy in more realistic control systems and demonstrate the extensibility of this approach to other high-degree systems.

\section{METHOD}
\subsection{Background}
A general affine control system takes the form $\dot{\textbf{x}} = f(\textbf{x})+g(\textbf{x})u $. Given an affine control system of this form where $f$ and $g$ are locally Lipschitz and $\textbf{x} \in D \subset \mathbb{R}^n$ is the state and $u \in U \subset \mathbb{R}^m$ is the set of admissible inputs \cite{ames2019control}. Let us define a function $h(\textbf{x}): D \subset \mathbb{R}^n \to \mathbb{R}$, where the set $C = \{\textbf{x} \in D\subset\mathbb{R}^n: h(\textbf{x}) \geq 0\}$ is called the safe set\cite{ames2019control}. This allows for the barrier function to be defined \cite{ames2019control}:
\begin{definition}
Let $C \subset D \subset \mathbb{R}^n$ be the superlevel set of a continuously differentiable function $h: D\to \mathbb{R}$, then $h$ is a control barrier function if there exists an extended class $\mathcal{K}_\inf$ function $\alpha$ such that for the affine control system 
$
\sup_{u\in U}\left[L_fh(\textbf{x}) + L_g(h(\textbf{x})u) \geq -\alpha(h(\textbf{x}))\right]
$
for all $\textbf{x} \in D$.
\end{definition}
This allows for the definition of the set of all inputs that render the set safe as stated in \cite{ames2019control} to be
$
	K_{\text{cbf}}(\textbf{x}) = \{u \in U: L_fh(\textbf{x}) + L_gh(\textbf{x})u + \alpha(h(\textbf{x})) \geq 0\}
$.

Since the definition of Control Barrier Functions only includes a first derivative, in order to use a higher-relative-degree barrier, a stronger condition will be needed. The eCBF, as proposed in \cite{ames2019control}, permits Control Barrier Functions of higher relative degree than one.

Let 
\begin{align}
	\eta_b = \begin{bmatrix}h(\textbf{x})\\ \dot{h}(\textbf{x}) \\ \vdots\\ h^{(r-1)}(\textbf{x})\end{bmatrix}
	F =
        \begin{bmatrix}
        0 & I_{r-1 \times r-1}\\
        \vdots\\
        0 & \dots0 \\
        \end{bmatrix}
	G = \begin{bmatrix}\textbf{0}_{1\times r} \\ 1\end{bmatrix}
\end{align}
where $r$ is the relative degree, $F$ is $r\times r$, $\eta_b$ is $1\times r$ and $G$ is $1\times r$.
Given this, the definition of an eCBF given in \cite{ames2019control} is:
\begin{definition}
	Given a set $C \subset D \subset \mathbb{R}^n$ defined as the superlevel set of a $r$-times continuously differentiable function $h: D \to \mathbb{R}$, then $h$ is an exponential Control Barrier Function if there exists a row vector $K_\alpha \in \mathbb{R}^r$ such that for the affine control system
	\begin{align}
		\sup_{u \in U} \left[ h^{(r)}(\textbf{x})\right] \geq -K_\alpha\eta_b(\textbf{x}) \label{ecbf}
	\end{align}
\end{definition}
Similar to CBFs, where $\alpha > 0$ is required to guarantee safety, there are some constraints on $K_\alpha$ required for (\ref{ecbf}) to guarantee safety. Let$
	v_0 = h(\textbf{x}),\;\;
	v_i = \dot{v}_{i-1} + p_iv_{i-1}$
where $p_i$ are the eigenvalues of $F-GK_\alpha$, then each eigenvalue $p_i$ where $1\leq i \leq r$ must satisfy $p_i>0 ,p_i \geq -\frac{\dot v_{i-1}(\textbf{x}_0, u_0)}{v_{i-1}(\textbf{x}_0, u_0)}$ \cite{ames2019control}, where $\textbf{x}_0$ and $u_0$ are the state and input at epoch $0$.
\begin{lemma}
The constraints that must be satisfied for safety under an eCBF are $\sup_{u \in U} \left[ h^{(r)}(\textbf{x})\right] \geq -K_\alpha\eta_b(\textbf{x})$, $p_i>0 ,p_i \geq -\frac{\dot v_{i-1}(\textbf{x}_0, u_0)}{v_{i-1}(\textbf{x}_0, u_0)}$ \cite{ames2019control}. 
\end{lemma}
\subsection{Probabilistic Problem Formulation}
Motion uncertainty is modelled by introducing a normal random vector into the affine control system as in \cite{lyu2021probabilistic} 
\vspace{-.35cm}
\begin{align}g_s(\textbf{x}) = \begin{bmatrix}g(\textbf{x})&I_{n\times n}\end{bmatrix}, u_s = \begin{bmatrix}u\\\epsilon_{n \times 1}\end{bmatrix} \end{align}
The stochastic controller takes the form
	\begin{align}\dot{\textbf{x}}_P = f(\textbf{x}_P)+g_s(\textbf{x}_P)u_s = f(\textbf{x}_P)+g(\textbf{x}_P)u + \epsilon\label{stoch_affine}\end{align}
Where $\epsilon$ is a normal random vector in $\mathbb{R}^n$ which introduces uncertainty into the state vector, $\textbf{x}_P$ . The first and third constraints from lemma $1$ will now contain random variables. Given this, the new constraints must be satisfied probabilistically. Consider the constraints (\ref{pecbf}) and (\ref{pKa}). Let the inputs that satisfy (\ref{pecbf}) and (\ref{pKa}) be $\mathcal{B}^s$ as in \cite{luo2020multi}, then, by the same proof (\ref{pecbf}) and (\ref{pKa}) ensure the satisfaction of the corresponding constraints from lemma 1 in a stochastic system with confidence $\eta$. 

Given these constraints we can construct a general optimization problem that ensures safety probabilistically and ensures feasibility while secondarily conforming to a desired $u$ and $K_\alpha$:
\vspace{-.2cm}
\begin{align}
\min_{\{u, K_\alpha\}}& c_1||u-u_{\text{desired}}|| + c_2||K_\alpha-K_{\alpha_{\text{desired}}}|| \label{gen_optim}\\
\text{s.t.}\quad &P\left( \sup_{u \in U} \left[ h^{(r)}(\textbf{x})\right] \geq -K_\alpha\eta_b(\textbf{x}) \right) \geq \eta \label{pecbf}\\
&p_i > 0, \;\;P\left(p_i \geq -\frac{\dot v_{i-1}(\textbf{x}, u)}{v_{i-1}(\textbf{x}, u)} \right) \geq \eta \label{pKa} 
\end{align}
Unlike the original eCBF formulation, we select $K_\alpha$ at each time step rather than only at epoch $0$. The constraints in the above optimization problem are both sufficient and necessary for probabilistic safety. Hence no safe input is excluded with these constraints, and they guarantee the feasibility of any safe input at each time step. 
% Also unlike CBFs, $K_\alpha$ and $u$ can be codependent. Since $h^{(a)}: a\geq b$, where $b$ is the smallest degree of $h$ relative to $d$ for $d \in u$, depends on $u$, $K_\alpha$ may depend on $u$, which in turn depends on $K_\alpha$. Therefore $K_\alpha$ will have to be optimized in the same program as $u$, we cannot simply minimize or maximize $K_\alpha$.
 
Explicit characterization of these constraints is omitted for the sake of generality, as they are dependent upon the barrier function and system dynamics. Detailed examples are given in the case studies. However, given that uncertainty is added as random variables to the existing variables in the dynamics that are not functions of $\textbf{x}$, we can characterize what the probabilistic constraints will be relative to the deterministic constraints.

\begin{theorem}{(Stochastic System Transition)}
For a deterministic system using eCBFs as described in \cite{ames2019control}, suppose the resulting constraints over the admissible control space (first constraint from lemma $1$) and $K_\alpha$ (third constraint from lemma $1$) take the form $f_{\text{ad}}(\dot{\textbf{x}}, u, K_\alpha, C) \geq 0$ and $f_{p_i}(\dot{\textbf{x}}, u, K_\alpha, C) \geq 0$ where $i \in \mathbb{N} \cap [1, r]$ and $C$ contains $\textbf{x}, \ddot{\textbf{x}}, \dots$. By introducing uncertainty as described above, the corresponding constraints (\ref{pecbf}) and (\ref{pKa}) for the stochastic system will be $P\left(f_{\text{ad}}(\dot{\textbf{x}} + \epsilon, u, K_\alpha, C) \geq 0 \right) \geq \eta$ and $P\left(f_{p_i}(\dot{\textbf{x}} + \epsilon, u, K_\alpha, C) \geq 0\right) \geq \eta$ where $i \in \mathbb{N} \cap [1, r]$. The transition from deterministic to stochastic systems is equivalent to simply adding $\epsilon$ to the resulting constraints after evaluation while preserving probabilistic safety. 
\end{theorem}
\begin{proof}
Note that $\dot{\epsilon} = \textbf{0}_{1\times n}$. If $\textbf{x} = \textbf{x}_P$, then $\dot{\textbf{x}}_P = \dot{\textbf{x}}+\epsilon$. Furthermore, $\ddot{\textbf{x}}_P = \ddot{\textbf{x}}$. It is clear that no further derivatives depend on $\epsilon$. Therefore, to calculate the constraints, $\textbf{x}$ can be used place of $\textbf{x}_P$ and $\textbf{x}^{(c)}$ for $\textbf{x}_P^{(c)}$ when $c\geq 2$, but $\dot{\textbf{x}}_P = \dot{\textbf{x}}+\epsilon$.
\end{proof}\vspace{-.1cm}
Given this, the probabilistic constraints are no harder to derive than the deterministic constraints. All that must be done is to use an inverse density function to account for the random variables.

While selection of $K_\alpha$ need only be done at the beginning of the control period in the determinstic case, we choose it at every time step to guarantee the feasibility of the eCBF, while also keeping the entries of $K_\alpha$ as close as possible to some desired values. However, if $K_\alpha$ is chosen once, it is guaranteed to remain valid only if the constraints on $u$ and $K_\alpha$ are satisfied at every epoch \cite{ames2019control}. Hence when they are violated due to the unbounded uncertainty, which we are $\eta$-confident will not happen for each constraint, $K_\alpha$ may need to be reset. Because of this, $K_\alpha$ optimization is needed for a safety guarantee. Similarly, if a deterministic barrier is used in a stochastic model while optimizing $K_\alpha$, the repulsive force of the boundary of the safe set is reduced because of the $K_\alpha$ optimization while the state stays at the equivalent of a confidence of $0.5$ because of the deterministic barrier. This would be a highly dangerous control method. It is for these reasons that probabilistic barriers and $K_\alpha$ optimization are best used together to achieve safe control. Employing both of these techniques will offer a probabilistic safety guarantee and a feasibility guarantee at each time-step.
\section{Applications
% : Probabilistic Barriers and Parameter Optimization for Robotic Systems
} 
\subsection{Case Study: Lane Changing}
\label{lanechanging} 
Here we take lane changing in a two-lane context using the kinematic bicycle model as an example. The kinematic bicycle model \cite{kongbicycle} can effectively be captured with the affine control system 
\begin{align}
\begin{split}
	\begin{bmatrix}
		\dot x\\
		\dot y\\
		\dot \psi\\
		\dot v\\
	\end{bmatrix}
	=
	&\begin{bmatrix}
		v\cos(\psi)\\
		v\sin(\psi)\\
		0\\
		0\\
	\end{bmatrix}
	+
\left[\begin{array}{@{}c c@{}}
  \begin{matrix}
    0&-v\sin(\psi)\\
    0&v\cos(\psi)\\
    0&v/l_r\\
    1&0
  \end{matrix}
  I_{4\times4} \\
  \end{array}\right]
	\begin{bmatrix}
		a\\
		\beta\\
		\epsilon_1\\
		\textbf{0}_{1 \times 3}
	\end{bmatrix}
	\end{split}
	\label{randbicycleaffine}
\end{align}
where $x$ and $y$ are position variables, $\psi$ is the inertial heading and $v$ is speed. $a$ is acceleration in the direction of current velocity and $\beta$ is slip angle \cite{he2021rulebased}. The small angle assumption \cite{he2021rulebased} is made to ensure the kinematic bicycle model is in the control affine form. To ensure pairwise vehicle safety, uncertainty along the direction of travel is taken into consideration.

The scenario is described in Fig \ref{diagram1}, and the goal of the ego vehicle is to maintain predefined safety margin $r$ from the surrounding vehicles in the direction of travel. The CBF that enforces pairwise vehicle safety is
$
	h_{m}(\textbf{x}) = (x_e-x_m)^2-r^2,
$
where $e$ denotes the ego vehicle, and $m$ denotes the other vehicle under consideration. Since position is considered, and acceleration is an input, this is a higher-relative-degree CBF, and it will be used as an exponential CBF.

\begin{figure}
	\centering
	\includegraphics[width=.5\linewidth]{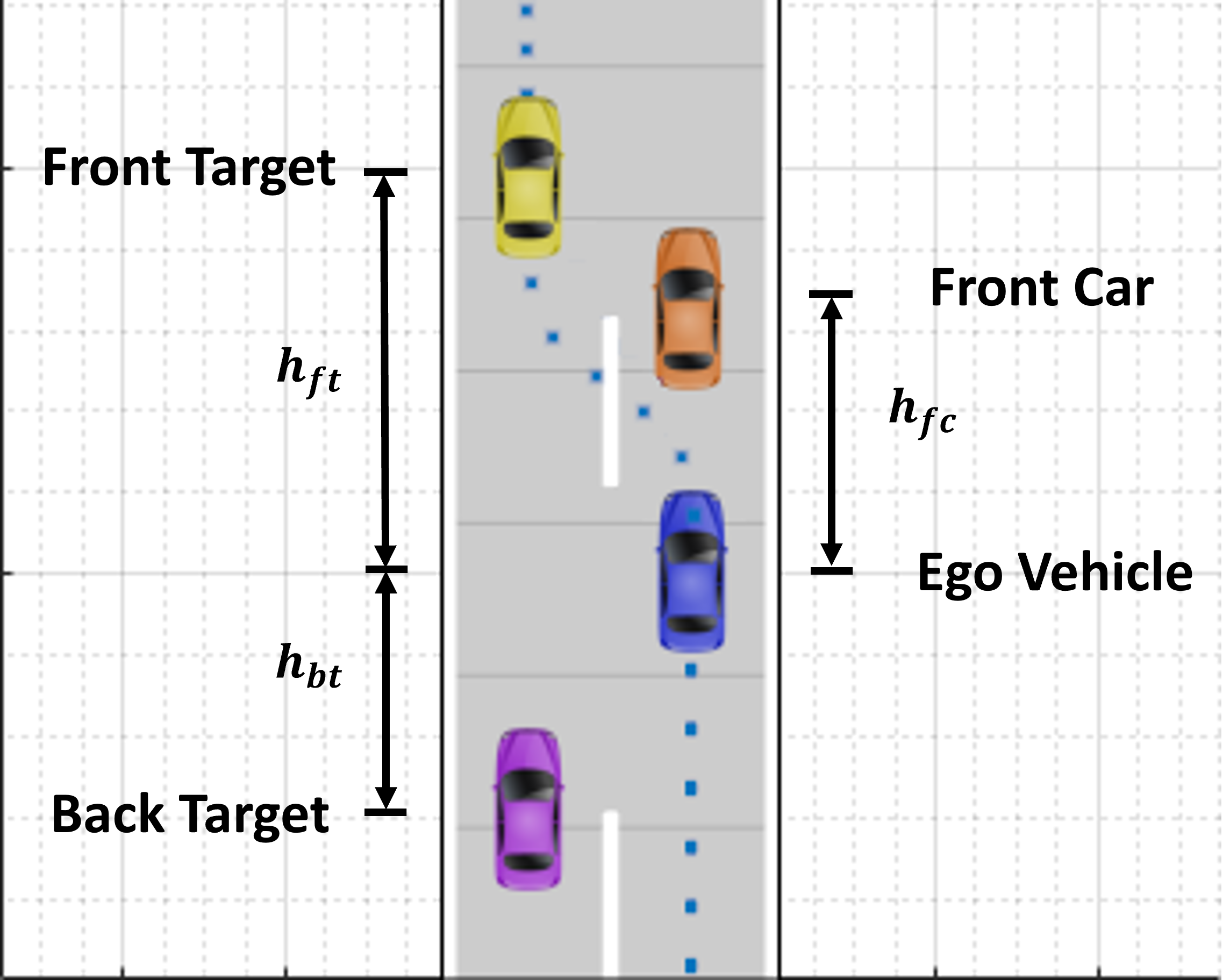}
	\caption{The lane changing scenario. The blue car is the ego vehicle, which uses a  Control Lyapunov Function-based nominal controller \cite{he2021rulebased}. $h_{fc}$,$h_{ft}$,$h_{bt}$ denote the safety function $h$ between the ego vehicle and the front car in the current lane, the front target in the target lane, and the back target in the target lane, respectively. The goal of the ego target is to follow the nominal path (in dotted line) to merge in between the front target and the back target in the adjacent lane. 
  }
	\label{diagram1}\vspace{-0.5cm}
\end{figure}

\begin{proposition}
	If $h_m$ is an exponential Control Barrier Function for the affine system (\ref{randbicycleaffine}), then the admissible control space $B(\textbf{x}, u)$\footnote{Appendix of detailed proof can be found at https://drive.google.com/file/d/1XKNQnhDRczJeXA3-XHkQgUKuHOwV4UFZ/view?usp=sharing.}
	\vspace{-0.15cm}
	 \begin{align}
		\begin{split}
			B(\textbf{x}, u) = &\left\lbrace u_e \in U_e \;:\; \frac{b_1-\bar \Delta\epsilon_1}{\sigma} \leq \Phi^{-1}(1-\eta)\right\rbrace  \cup\\ &\left\lbrace u_e \in U_e \;:\; \frac{b_2-\bar \Delta\epsilon_1}{\sigma} \geq \Phi^{-1}(\eta)\right\rbrace 
		\end{split}
	\end{align}
	where $\epsilon_{1_e} - \epsilon_{1_0} \sim \mathcal{N}(\bar{\Delta}_{\epsilon_1}, \sigma)$, $b$ is a state-dependent variable, and $\Phi^{-1}$ is the inverse standard normal cumulative distribution function (CDF), will ensure the continued safety of $h_m$ with confidence $\eta$.
\end{proposition}
\begin{proposition}
    Let $j = -p_1((\Delta_{x})^2-r^2) - 2\Delta_x\Delta_{\dot x}$.
	If $h_m$ is an exponential Control Barrier Function for the affine system (\ref{randbicycleaffine}), then the admissible space for valid $K_\alpha$ choice is
	\vspace{-0.15cm}
	\begin{align}
		p_1& > 0, p_2>0 \label{positive_p}\\
		&\begin{cases}
			\begin{aligned}
				&j - 2\Delta_x\bar\Delta_{\epsilon_1} - \Phi^{-1}(1-\eta)\sigma2\Delta_x \leq 0 & \Delta_x > 0\\
				&j - 2\Delta_x\bar\Delta_{\epsilon_1} - \Phi^{-1}(\eta)\sigma2\Delta_x \leq 0& \Delta_x < 0\\
				&p_1r^2 \leq 0& \Delta_x=0
			\end{aligned}
		\end{cases}\label{prob_ka1}\\
		&
		\begin{aligned}
				\left\lbrace \frac{d_1-\bar \Delta\epsilon_1}{\sigma} \leq \Phi^{-1}(1-\eta)\right\rbrace  \lor \left\lbrace \frac{d_2-\bar \Delta\epsilon_1}{\sigma} \geq \Phi^{-1}(\eta) \right\rbrace
		\label{prob_ka2} 
		\end{aligned}
	\end{align}
where $\epsilon_{1_e} - \epsilon_{1_0} \sim \mathcal{N}(\bar{\Delta}_{\epsilon_1}, \sigma)$, $\Delta_x = x_e-x_m$, $\Delta_{\dot{x}} = \dot{x}_e-\dot{x}_m$, $d$ is a state dependent variable, and $\Phi^{-1}$ is the inverse standard normal CDF.
\end{proposition}

To evaluate the validity and effectiveness of our method, four baseline methods are chosen for comparison: \textbf{(Baseline 1)} a traditional first-degree CBF from \cite{he2021rulebased} that uses explicit physics calculations to account for acceleration, \textbf{(Baseline 2)} the proposed method without probabilistic extension and $K_\alpha$ optimization, \textbf{(Baseline 3)} the proposed method without probabilistic extension but with $K_\alpha$ optimization, \textbf{(Baseline 4)} the proposed method with probabilistic extension but without $K_\alpha$ optimization.

The task performance in terms of collision rate and feasibility satisfaction is compared among the proposed method and the baseline methods. All the methods share the same designs in nominal controller and vehicle dynamic model, which ensures the fairness of comparison. All surrounding vehicles move at a randomly generated constant speed.

 The results on 250 randomly generated lane changing scenarios are shown in Fig. \ref{experiment1}. The starting positions and velocities of each vehicle are generated with noise from a uniform distribution. In all scenarios the initial position of the ego vehicle is always set behind the front car and the front target and in front of the back target, as illustrated in Fig. \ref{diagram1}. It is observed that the proposed method achieved the best overall performance with lowest collision and infeasibility rate and the highest successful rate with no unfinished cases. Simulation is done in Matlab, and code is available on Github.\savefootnote{github}{https://github.com/SpencerKoevering/eCBFKaOptimization} 

It is observed that baselines 1, 2, and 4 all result in a higher infeasibility rate and more unfinished cases compared to the proposed method due to the lack of $K_\alpha$ optimization, which supports our theoretical claim that $K_\alpha$ optimization can expand the solution space whenever possible. Both the proposed method and baseline 3 employ $K_\alpha$ optimization, and the only difference is that the proposed method uses the probabilistic extension to take uncertainty into consideration while baseline 3 does not. Therefore, the proposed method is safer in terms of lower collision rate. In addition, our method has greater generality and robustness than baseline 1, as our method is not dependent upon additional physics assumptions.
\begin{figure}
	\centering
	\includegraphics[width=1\linewidth]{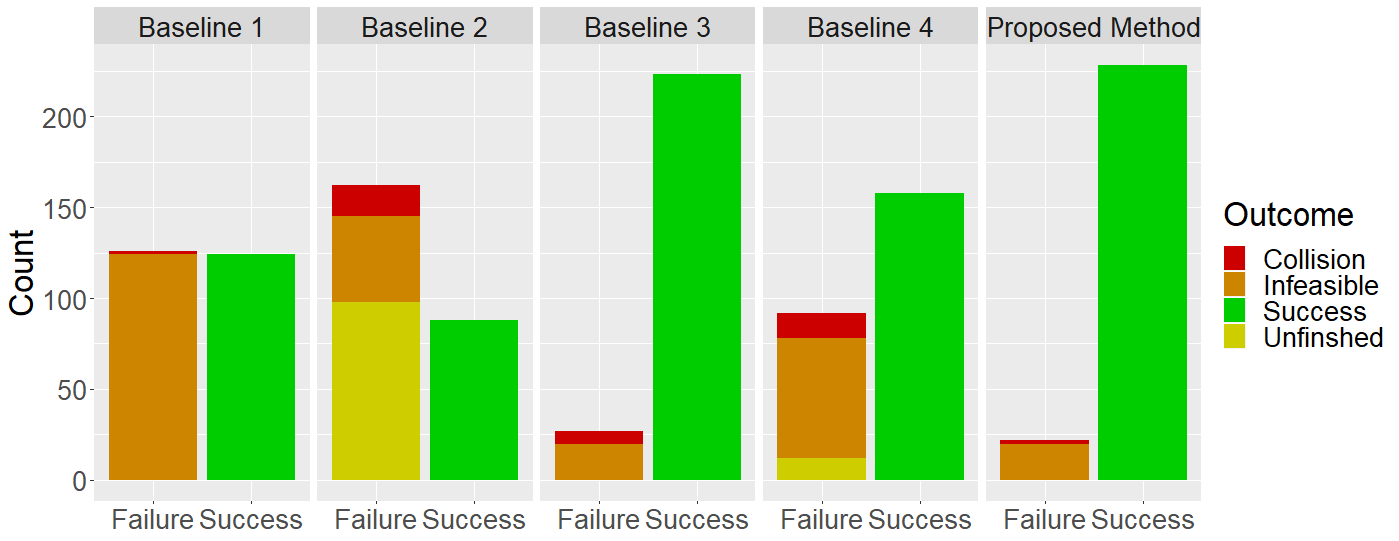}
	\caption{Comparison of the proposed methods with 4 baseline methods. Here $\epsilon ~ \mathcal{N}(0, 0.15)$. The confidence level $\eta$ is set to be $0.99$ for the probabilistic extension. An infeasible outcome indicates that the controller could not find a safe input, this often leads to collision in the next time step if a backup controller is used. An unfinished outcome means that the ego vehicle was unable to change lanes, but never collided or became infeasible. A successful outcome indicates that the ego vehicle changed lanes without colliding or becoming infeasible.
	}
	\label{experiment1}\vspace{-0.5cm}
\end{figure}
\subsection{Case Study: Intersections}

The second case study presented uses eCBFs to ensure safe traversal of an uncontrolled intersection in multiple scenarios. Similar to section \ref{lanechanging}, the kinematic bicycle model is used to describe vehicle dynamics, though uncertainty is considered in both the $x$ and $y$ directions in this case, which allows for higher-dimensional modelling.

The motivation for including this case study is that: 1) In this case, unlike section  \ref{lanechanging}, the uncertainty modeling is high-dimensional, which is a more realistic general setup. 2) To handle the high-dimensional modelled uncertainty, a principled theoretical derivation of the resulting safety constraints is provided while a practical means of approximating the uncertainty is also provided to ensure computational efficiency.
\vspace{-0.15cm}
\begin{align}
\begin{split}
		\begin{bmatrix}
			\dot x\\
			\dot y\\
			\dot \psi\\
			\dot v\\
		\end{bmatrix}
		=
		&\begin{bmatrix}
			v\cos(\psi)\\
			v\sin(\psi)\\
			0\\
			0\\
		\end{bmatrix}
		+
		\left[\begin{array}{@{}c c@{}}
          \begin{matrix}
            0&-v\sin(\psi)\\
            0&v\cos(\psi)\\
            0&v/l_r\\
            1&0
          \end{matrix}
          I_{4\times4} \\
         \end{array}\right]
		\begin{bmatrix}
			a\\
			\beta\\
			\epsilon_1\\
			\epsilon_2\\
		\textbf{0}_{1 \times 2}
		\end{bmatrix}
		\end{split}
		\label{randbicycleaffineintersection}
\end{align}

In this case, a different CBF $h$ is needed. The main reason is that the analogous CBF to the one used in section \ref{lanechanging}: $h_o(x) = (x_e-x_o)^2+(y_e-y_o)^2+r^2$, results in a quadratic form of $\epsilon$ in the constraint. However, a quadratic form of two or more normal random variables cannot be solved with the quadratic formula and an inverse normal distribution.
A generalized chi square inverse CDF could be approximated to satisfy these constraints\cite{das2021method}.  However, this approximation is not computationally efficient enough to be deployed in real-time applications. Instead, we approximate the $2$-norm with a $1$-norm, to prevent the quadratic form of $\epsilon$. Using the $1$-norm, this constraint can be satisfied using the algebraic properties of the normal distribution. 

As a result, the CBF used is
 \begin{align}
 \begin{split}
		h_o(x) =& \left|x_e-x_o\right| - b_{\{x,e\}} - b_{\{x,o\} } + \left|y_e-y_o\right|\\ &- b_{\{y,e\}} - b_{\{y,o\}}-r \label{2dh}
\end{split}
\end{align}
where $e$ denotes the ego vehicle and $o$ denotes the other vehicle under consideration. $x$ and $y$ represent position. To offset the precision loss from a $1$-norm, bounding boxes are used to calculate minimum distance. Assuming that the vehicles are rectangular, then for each vehicle there is some value $b_x$ which is the distance to the farthest point on the edge of the vehicle from the center of mass in the $\pm x$-direction. Similarly, $b_y$ is this greatest distance in the $\pm y$-direction. $r$ represents any extra safety distance that is desired.
\begin{figure}
	\centering
    \includegraphics[width=.45\linewidth]{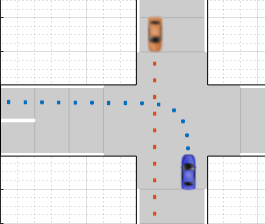}
 	\includegraphics[width=.45\linewidth]{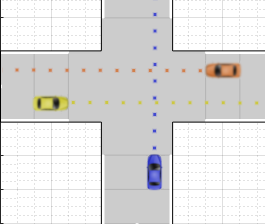}
	\caption{Intersection experiment setup. Scenario 1 (left) and scenario 2 (right). Dotted lines indicate paths. The ego vehicle is blue.}
	\label{diagram3}\vspace{-0.5cm}
\end{figure}

A nonlinear MPC-based nominal controller \cite{nlmpc} is used to calculate $u_{\text{desired}}$ for all vehicles.  Two intersection scenarios are considered, as shown in Fig. \ref{diagram3}, where the ego vehicle attempts to navigate an uncontrolled intersection of two-lane roads. In the first scenario the ego vehicle attempts a left turn while there is an oncoming vehicle, and in the second the ego vehicle attempts to travel straight through the intersection while there are crossing vehicles, one from each direction. The other vehicles do not attempt to avoid collisions. Simulation is done in Matlab, and code is available on Github.\repeatfootnote{github}
% \footnotemark[\value{footnote}]

\begin{figure}
	\centering
    \includegraphics[width=1\linewidth]{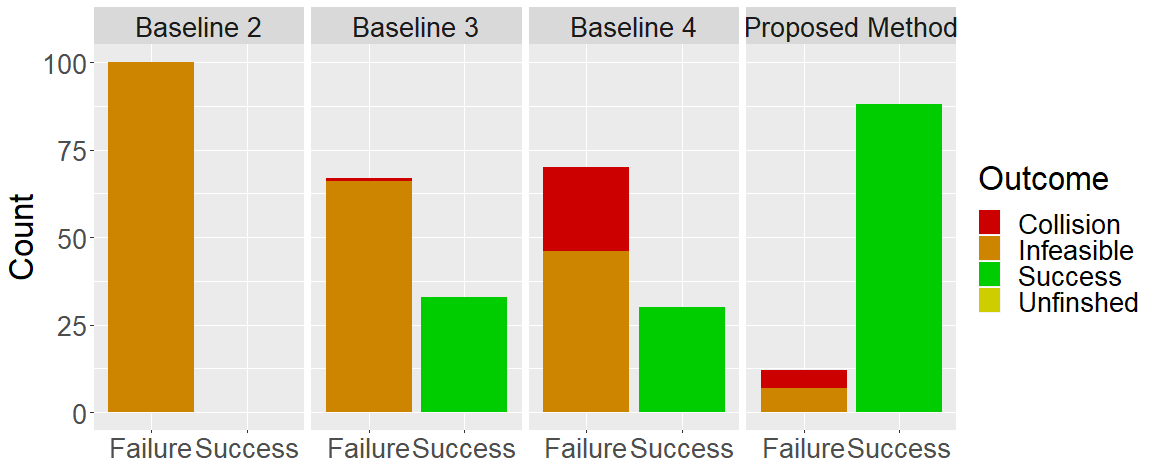}
    \includegraphics[width=1\linewidth]{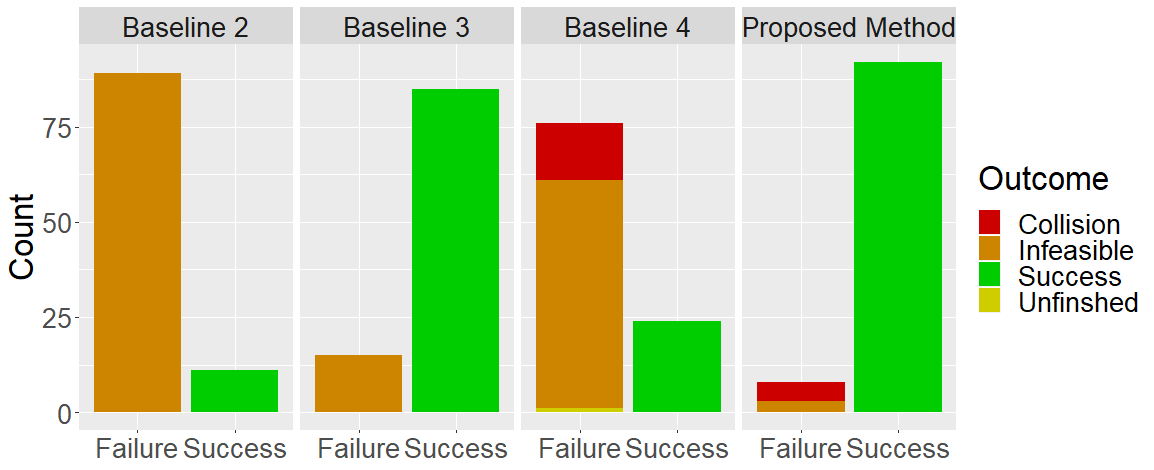}
	\caption{Comparison of the proposed methods with 3 baseline methods in two intersection scenarios. 
	Here $\epsilon ~ \mathcal{N}(0, 0.15)$. The confidence level is $\eta=0.9999$. Weights in the objective function for the optimization problem are adjusted for each scenario. In case 2 the nlmpc avoids being between the other two vehicles. }
	\label{experiment3o}\vspace{-0.5cm}
\end{figure}

Similar to the last application, the proposed method achieves the best overall performance in terms of the highest success rate in both of the scenarios, compared to the baseline methods. We also notice some interesting comparisons among the baseline methods. 
The authors argue that it is important to equip the system with both  probabilistic extension and $K_\alpha$ optimization at the same time, especially in environments with higher-dimensional uncertainty.
In the first scenario, more collisions with baseline 3 are observed than baseline 2, which indicates a negative impact of $K_\alpha$ optimization on safety in the deterministic controller. Although baseline 4 is equipped with probabilistic extension to account for uncertainty, without $K_\alpha$ optimization, it results in a higher collision rate than any other controller.
In the second scenario, some similar observations are obtained. It is hard to tell whether the issue of collision or infeasibility is more fatal in the real world, since in most cases, once infeasibility arises, no safe guidance can be provided to the system. The deterministic barriers did better in this scenario, likely due to the fact that in this case they are not required to get close to the other vehicles due to perpendicular travel. All these observations support the argument that both the probabilistic extension and $K_\alpha$ optimization should be deployed at the same time in higher-dimensional environments, yet the deeper reasons behind these observations remain for future study.\vspace{-0.1cm}

\section{CONCLUSION}
We present a novel safe control method for autonomous vehicles with higher-relative degree dynamics under unbounded uncertainty. Our method provides provable probabilistic guarantees for safety and feasibility. Simulation results demonstrate the effectiveness of our method compared to other existing methods. Future work includes extending experiments on other robotic platforms with higher-relative-degree dynamics.

%%%%%%%%%%%%%%%%%%%%%%%%%%%%%%%%%%%%%%%%%%%%%%%%%%%%%%%%%%%%%%%%%%%%%%%%%%%%%%%%
% \section*{ACKNOWLEDGEMENT}
% The authors would like to thank the organizers of the CMU Robotics Institute Summer Scholars program for making this work possible. This work was sponsored by the NSF Research Experience for Undergraduates Program.
\nocite{*}  % Without this, cite articles in text using \cite{...}
\bibliographystyle{IEEEtran}
\bibliography{./IEEEfull,refs}

\end{document}